\documentclass{article}
\usepackage{ijcai17}
\usepackage{times}
\usepackage{courier}
\usepackage{amsthm,amsmath,amssymb}
\usepackage{graphicx}
\usepackage{float}
\usepackage{multirow}
\usepackage{booktabs}
\usepackage{color}
\usepackage{caption}
\usepackage{subcaption}
\usepackage{grffile}
\usepackage{bm}
\usepackage{mathtools}
\usepackage{algorithm}
\usepackage[noend]{algpseudocode}
\usepackage[colorlinks=true, allcolors=blue]{hyperref}

\DeclareRobustCommand{\citep}[1]{\cite{#1}}
\DeclareRobustCommand{\citet}[1]{\citeauthor{#1}~\shortcite{#1}}

\newtheorem{theorem}{Theorem}
\newtheorem{corollary}{Corollary}[theorem]

\def\bzero{{\boldsymbol 0}}

\def\bt{{\boldsymbol t}}
\def\bw{{\boldsymbol w}}
\def\bx{{\boldsymbol x}}
\def\by{{\boldsymbol y}}
\def\bz{{\boldsymbol z}}

\def\bgamma{{\boldsymbol \gamma}}
\def\bbeta{{\boldsymbol \beta}}

\def\bZ{{\boldsymbol Z}}

\def\cH{{\mathcal H}}

\def\cX{{\mathcal X}}

\def\bE{{\mathbb E}}
\def\bP{{\mathbb P}}
\def\bR{{\mathbb R}}

\def\bI{{\mathbf I}}

\newcommand{\hk}{\hat{k}}
\newcommand{\tk}{\tilde{k}}
\newcommand{\Rcal}{{\mathcal{R}}}

\def\SingleWidth{0.23\textwidth}
\def\DoubleWidth{0.325\textwidth}

\begin{document}

\title{Data-driven Random Fourier Features using Stein Effect}
\author{Wei-Cheng Chang\\
  LTI, CMU\\
  wchang2@cs.cmu.edu\\
  \And
  Chun-Liang Li\\
  MLD, CMU\\
  chunlial@cs.cmu.edu\\
  \And
  Yiming Yang\\
  LTI, CMU\\
  yiming@cs.cmu.edu\\
  \And
  Barnab{\'a}s P{\'o}czos \\
  MLD, CMU\\
  bapoczos@cs.cmu.edu\\
}

\maketitle
\begin{abstract}
Large-scale kernel approximation is an important problem in machine learning research. Approaches using random Fourier features have become increasingly popular \cite{Rahimi_NIPS_07}, where kernel approximation is treated as empirical mean estimation via Monte Carlo (MC) or Quasi-Monte Carlo (QMC) integration \cite{Yang_ICML_14}. A limitation of the current approaches is that all the features receive an equal weight summing to 1.
In this paper, we propose a novel shrinkage estimator from "Stein effect", which provides a data-driven weighting strategy for random features and enjoys theoretical justifications in terms of lowering the empirical risk. We further present an efficient randomized algorithm for large-scale applications of the proposed method.  
Our empirical results on six benchmark data sets demonstrate the advantageous performance of this approach over representative baselines in both kernel approximation and supervised learning tasks.
\end{abstract}

\vspace{-5pt}
\section{Introduction}
Kernel methods offer a comprehensive collection of non-parametric techniques for modeling a broad range of problems in machine learning. 
However, standard kernel machines such as kernel SVM or kernel ridge regression suffer from slow training and prediction which limits their use in real-world large-scale applications. 
Consider, for example, the training of linear regression over $n$ labeled data points~$\{(\bx_i,y_i)\}_{i=1}^n$ where $\bx_i \in \mathbb{R}^d$ with~$n \gg d$ and $y_i$ is a binary label. 
The time complexity of standard least square fit is~$O(nd^2)$ and the memory demand is $O(nd)$. 
Its kernel-based counterpart requires solving a linear system with an $n$-by-$n$ kernel matrix, which takes $O(n^3 + n^2d)$ time and $O(n^2)$ memory usage as typical. Such complexity is far from desirable in big-data applications.

In recent years, significant efforts have been devoted into low-rank kernel approximation for enhancing the scalability of kernel methods.  Among existing approaches, random features \cite{Rahimi_NIPS_07,Rahimi_NIPS_09,Le_ICML_13,Bach_arXiv_15} have drawn considerable attention and yielded state-of-the-art results in classification accuracy on benchmark data sets \cite{Huang_ICASSP_14,Dai_NIPS_14,Li_UAI_16}. Specifically, inspired from Bochner's theorem \cite{Rudin_book_11}, random Fourier features have been studied for evaluating the expectation of shift-invariant kernels (i.e., $k(\bx,\bx')=g(\bx-\bx')$ for some function $g$). \citet{Rahimi_NIPS_07} proposed to use Monte-Carlo methods (MC) to estimate the expectation; \citet{Yang_ICML_14} leveraged the low-discrepancy properties of Quasi-Monte Carlo (QMC) sequences to reduce integration errors.

Both the MC and QMC methods rely on the (implicit) assumption that all the $M$ random features are equally important, and hence assign a uniform weight $\frac{1}{M}$ to the features in kernel estimation.  Such a treatment, however, is arguably sub-optimal for minimizing the expected risk in kernel approximation. \citet{Avron_JMLR_16} presented a weighting strategy for minimizing a loose error bound which depends on the maximum norm of data points.
On the other hand, Bayesian Quadrature (BQ) is a powerful framework that solves the integral approximation by $\bE[f(x)] \approx \sum_m \beta_m f(x_m)$ where~$f(x_m)$ is feature function and $\beta_m$ is with non-uniform weights. BQ leverages Gaussian Process (GP) to model the prior of function $f(\cdot)$, and the resulted estimator is Bayes optimal~\cite{Huszar_UAI_12}.  Thus, BQ could be considered a natural choice for kernel approximation. Nonetheless, a well-known limitation (or potential weakness) is that the performance of Bayesian-based models heavily relies on extensive hyper-parameter tuning, e.g., on the condition of the covariance matrix in the Gaussian prior, which is not suitable for large real-world applications. 

To address the fundamental limitations of the aforementioned work, we propose a novel approach to data-driven feature weighting in the approximation of shift-invariant kernels, which motivated by the by Stein Effect in the statistical literature (Section \ref{sec:Stein}), and solve it using an efficient stochastic algorithm with a convex optimization objective.  
We also present a natural extension of BQ to the applications of kernel approximation (Section \ref{sec:BQ}). 
The adapted BQ together with standard MC and QMC methods serve as representative baselines in our empirical evaluations on six benchmark data sets.  
The empirical results (Section \ref{sec:exp}) show that the proposed Stein-Effect Shrinkage (SES) estimator consistently outperforms the baseline methods in terms of lowering the kernel approximation errors, and was competitive to or better than the best performer among the baseline methods in most task-oriented evaluations (on supervised learning of regression and classification tasks).

\iffalse
Addressing the fundamental limitations of the aforementioned work, we propose a simple yet effective weighting technique on the existing random feature to further improve the kernel approximation. Specifically, the main contributions in the proposed method include:
\begin{itemize}
    \item a new formulation from the risk minimization perspective for kernel approximation that takes the Stein Effect into account in the optimization of feature weights;
    \item an efficient stochastic algorithm for solving the shrinkage estimator with a convex objective function;
    \item a comparative evaluation of the proposed method and other representative methods (MC, QMC and BQ) on benchmark datasets, where our method consistently outperformed the other methods respect to lowering the errors of kernel approximation and supervised learning tasks.
\end{itemize}
\fi

\section{Preliminaries}
Let us briefly outline the related background and key concepts in random-feature based kernel approximation.

\subsection{Reproducing Kernel Hilbert Space}
At the heart of kernel methods lies the idea that inner products in 
high-dimensional feature spaces can be computed in an implicit form 
via a kernel function $k : \cX \times \cX \rightarrow \bR$, which is defined
on a input data domain $\cX \subset \bR^d$ such that
\begin{equation*}
	k(\bx,\bx') = \langle \phi(\bx), \phi(\bx') \rangle_{\cH}.
\end{equation*}
where $\phi: \cX \rightarrow \cH$ is a feature map that associates kernel $k$
with an embedding of the input space into a high-dimensional Hilbert space $\cH$.
A key to kernel methods is that
as long as kernel algorithms have access to $k$, one need not to 
represent~$\phi(\bx)$ explicitly. Most often, that means although 
$\phi(\bx)$ can be high-dimensional or even infinite-dimensional, their
inner products, can be evaluated by $k$. This idea is known as the 
kernel trick.

The bottleneck of kernel methods lies in both training and prediction phase. The former requires to compute and store the kernel matrix $K \in \bR^{n \times n}$ for $n$ data points, while the latter need to evaluate the decision function $f(\bx)$ via kernel trick $f(\bx) = \sum_{i=1}^N \alpha_i k(\bx_i,\bx)$, which sums over the nonzero support vectors $\alpha_i$. Unfortunately, \citet{Steinwart_book_08} show that the number of nonzero support vectors growth linearly in the size of training data $n$. This posts a great challenge for designing scalable algorithms for kernel methods in large-scale applications. Some traditional way to address
this difficulty is by making trade-offs between time and space. \cite{Kivinen_SP_04,Chang_TIST_11}.
%For example, LIBSVM \cite{Chang_TIST_11} only caches some columns of 
%the kernel matrix to save the memory usage and re-compute the columns
%where there is a cache miss.

\subsection{Random Features (RF)}
\citet{Rahimi_NIPS_07} propose a low-dimensional feature map 
$z(\bx): \cX \rightarrow \bR^{2M}$ for the kernel function $k$
\begin{equation}
	k(\bx,\bx') \approx \langle z(\bx), z(\bx') \rangle	\label{eq:kernel-approx}
\end{equation}
under the assumption that $k$ fall in the family of shift-invariant kernel.
The starting point is a celebrated result that characterizes 
the class of positive definite functions:
\begin{theorem}[Bochner's theorem \cite{Rudin_book_11}]
	A continuous, real valued, symmetric and shift-invariant function
	$k$ on $\bR^d$ is a positive definite kernel if and only if 
	there is a positive finite measure $\bP(\bw)$ such that
	\begin{align}
		k(\bx-\bx') 
		= \int_{\bR^d}& 2\big[ \cos(\bw^T\bx)\cos(\bw^T\bx') \label{eq:bochner} \\
					&+ \sin(\bw^T\bx)\sin(\bw^T\bx') \big] d\bP(\bw). \nonumber
	\end{align}
\end{theorem}
\noindent The most well-known kernel that belongs to the shift-invariant
family is Gaussian kernel 
$k(\bx,\bx') = \exp( -\frac{\|\bx-\bx'\|^2}{2\sigma^2} )$,
the associated density $\bP(\bw)$ is again Gaussian, $N(\bzero, \sigma^{-2}\bI_d)$.

\citet{Rahimi_NIPS_07} approximate 
the integral representation of the kernel \eqref{eq:bochner} 
by Monte Carlo method as follows,
\begin{equation}
	k(\bx-\bx')
    \approx \frac{1}{M} \sum_{m=1}^M h_{\bx,\bx'}(\bw_m) 
   	= \langle z(\bx), z(\bx') \rangle,  \label{eq:kernel-mc-approx}
\end{equation}
where 
\begin{align}
	h_{\bx,\bx'}(\bw) &\coloneqq  2\big[ \cos(\bw^T\bx)\cos(\bw^T\bx')+\sin(\bw^T\bx)\sin(\bw^T\bx') \big] \nonumber \\
    z(\bx) &\coloneqq \frac{1}{\sqrt{M}} [\phi_{\bw_1}(\bx),\ldots,\phi_{\bw_M}(\bx)] \label{eq:rf} \\
    \phi_{\bw}(\bx) &\coloneqq \sqrt{2}[\cos(\bw^T\bx), \sin(\bw^T\bx)] \nonumber
\end{align}
and $\bw_1,\ldots,\bw_M$ are i.i.d. samples from $\bP(\bw)$.

After obtaining the randomized feature map $z$, training data could be transformed into
$\{(z(\bx_i),y_i\}_{i=1}^n$. As long as $M$ is sufficiently smaller than $n$, 
this leads to more scalable solutions, e.g., for regression we get back to $O(nM^2)$ 
training and $O(Md)$ prediction time, with $O(nM)$ memory requirements. 
We can also apply random features to unsupervised learning problems, such as kernel clustering ~\cite{Chitta_ICDM_12}.

\subsection{Quasi-Monte Carlo Technique}
Instead of using plain MC approximation, \citet{Yang_ICML_14} propose to use the low-discrepancy properties of Quasi-Monte Carlo (QMC) sequences to reduce the integration error in approximations of the form \eqref{eq:kernel-mc-approx}. Due to space limitation, we restrict our discussion to the background that is necessary for understanding subsequent sections. We refer interested readers to the comprehensive reviews \cite{Caflisch_Acta_98,Dick_Acta_13} for more detailed exposition.

The QMC method is generally applicable to integrals over a unit cube. So the procedure is to first generate a
discrepancy sequence $\bt_1,\ldots,\bt_M \in [0,1]^d$, and transform it into a sequence $\bw_1,\ldots,\bw_M$ in $\bR^d$.
To convert the integral presentation of kernel \eqref{eq:bochner} to an integral over the unit cube, a simple change of
variables suffices. For $\bt \in \bR^d$, define~$\Phi^{-1}(\bt) = [\Phi_1^{-1}(t_1),\ldots,\Phi_d^{-1}(t_d)] \in \bR^d$,
where we assume that the density function in \eqref{eq:bochner} can be written as~$p(\bw) = \prod_{i=1}^M p_j(\bw_j)$ and $\phi_j$ is the cumulative distribution function (CDF) of $p_j, j=1,\ldots,d$. By setting $\bw = \Phi^{-1}(\bt)$, the integral \eqref{eq:bochner} is equivalent to 
\begin{equation*}
	k(\bx-\bx') 
    =\int_{\bR^d} h_{\bx,\bx'}(\bw)p(\bw)d\bw
    = \int_{[0,1]^d} h_{\bx,\bx'}( \Phi^{-1}(\bt)) d\bt.
\end{equation*}

\section{Bayesian Quadrature for Random Features}
\label{sec:BQ}
Random features constructed by either MC \cite{Rahimi_NIPS_07} or QMC \cite{Yang_ICML_14} approaches employ equal weights $1/M$ on integrand functions as shown in \eqref{eq:kernel-mc-approx}. An important question is how to construct different weights on the integrand functions to have a better kernel approximation.

Given a fixed QMC sequence, \citet{Avron_JMLR_16} solve the weights by optimizing a derived error bound based on the observed data.  There are two potential weakness of \cite{Avron_JMLR_16}.  First, \citet{Avron_JMLR_16} does not fully utilize the data information because the optimization objective only depends on the upper bound information of $|\bx_i|$ instead of the distribution of $\bx_i$.  Second, the error bound used in \citet{Avron_JMLR_16} is potentially loose. Due to technical reasons, \citet{Avron_JMLR_16} relaxes $h_{\bx,\bx'}(\bw)$ to $h_{\bx,\bx'}(\bw)\mbox{sinc}(T\bw)$ when they derive the error bound to optimize. It is not studied whether the relaxation results in a proper upper bound.

On the other hand, selecting weights by such way is closely connected to the Bayesian Quadrature (BQ), originally proposed by \citet{Ghahramani_NIPS_02} and theoretically guaranteed under Bayes assumptions. BQ is a Bayesian approach that puts prior knowledge on functions to solve the integral estimation problem. We first introduce the work of \citet{Ghahramani_NIPS_02}, and then discuss how to apply BQ to kernel approximations.

Standard BQ considers a single integration problem with form
\begin{equation}
	\bar{f} = \int f(\bw) d \bP(\bw).
	\label{eq:bq_single}
\end{equation}
To utilize the information of $f$, BQ puts a Gaussian Process (GP) prior on $f$
with mean $0$ and covariance matrix $K_{GP}$ where  
\begin{equation*}
	K_{GP}(\bw,\bw') 
    = Cov( f(\bw), f(\bw') ) 
    = \exp(-\frac{\|\bw-\bw'\|_2^2}{2\sigma_{GP}^2}).
\end{equation*}
After conditioning on sample $\bw_1,\cdots,\bw_M$ from $\bP$, we obtain a closed-form GP posterior over $f$ and use the posterior to get the
optimal Bayes estimator $\hat{f}$\footnote{For the square loss.} as
\begin{align*}
		\hat{f} 
        &= \bE_{GP}\left[\int f(\bw)d\bP(\bw) \right] \\ 
		&= \bgamma^\top K_{GP}^{-1}f(W) 
		=  \sum_{m=1}^M \beta_{BQ}^{(m)}f(\bw_m), 
\end{align*}
where $\gamma_m = \int K_{GP}(\bw, \bw_m)d\bP(\bw)$, 
$\beta_{BQ}^m = (K_{GP}^{-1} \bgamma)_m$,
and $f(W)=[f(\bw_1),\cdots,f(\bw_M)]$. 
Clearly, the resulting estimator is simply a new weighted
combination of the empirical observations 
$f(\bw_m)$ by $\beta_{BQ}^{(m)}$ derived from the Bayes perspective.

In kernel approximation, we could treat it as a series of the single integration problem. Therefore, we could directly apply BQ on it. Here we discuss the technical issues of applying BQ on kernel approximation. First, we need to evaluate~$\bgamma$. For Gaussian kernel, we derive the closed-form expression
%\begin{equation*}
%	\gamma_m 
%    = \int K_{GP}(\bw,\bw_m)p(\bw)d\bw
%    \propto N_{\bw_m}\Big( 0, (\sigma_{GP}^2+\sigma^{-2})\bI_M \Big),
%\end{equation*}
\begin{equation*}
	\gamma_m 
    = \int K_{GP}(\bw,\bw_m)\bP(\bw)
    = \exp\big( \frac{-\|\bw_m - \bzero\|_2^2}{\sigma_{GP}^2 + \sigma^{-2}} \big)
\end{equation*}
where $\gamma_m$ is a random variable evaluated at $\bw_m$.
%where $N_{\bw_m}(0,(\sigma_{GP}^2+\sigma^{-2})\bI_M)$ denotes the probability density function of Gaussian distribution with mean $0$ and covariance $(\sigma_{GP}^2+\sigma^{-2})\bI_M$ evaluated at $\bw_m$.

\subsection{Data-driven by Kernel Hyperparameters}
The most import issue of applying BQ on kernel approximation is that it models the information of $h_{\bx,\bx'}$ by the kernel $K_{GP}$. In practice, we usually use Gaussian kernel for $K_{GP}$, then the feature information of $(\bx, \bx')$ is modeling by the kernel bandwidth. As suggested by \citet{Rasmussen_GP_06}, one could tune the best hyper-parameter either by maximizing the marginal likelihood or cross-validation based on the metrics in interested. Note that the existing BQ approximates the integral representation of a kernel function value, evaluated at a single pair of data $(\bx,\bx')$. However, it is unfeasible to tune individual $\sigma_{GP}$ for $n^2$ pairs of data. 
Therefore, we tune a global hyper-parameter $\sigma_{GP}$ on a subset of pair data points.

\section{Stein Effect Shrinkage (SES) Estimator}\label{sec:Stein}
In practice, it is well known that the performance of Gaussian Process based models heavily depend on hyper-parameter $\sigma_{GP}$, i.e. the bandwidth of kernel function $k_{GP}$. Therefore, instead of using Bayesian way to obtain the optimal Bayes estimator from BQ, we start from the risk minimization perspective by considering the \emph{Stein effect} to derive the \emph{shrinkage estimator} for random features.

The standard estimator of Monte-Carlo methods is the unbiased empirical average using equal weights $\frac{1}{M}$ for $M$ samples which sum to 1. However, \emph{Stein effect} (or James-–Stein effect) suggests a family of biased estimators could result in a lower risk than the standard empirical average. Here we extend the analysis of non-parametric Stein effect to the kernel approximation with random features by considering the risk function $\Rcal(k, k') = \bE_{\bx,\bx',\bw}( k(\bx-\bx')-k'(\bx-\bx';\bw) )^2$.

\begin{theorem} \label{thm:stein} (Stein effect on kernel approximation)
Let $\hk(\bx-\bx', \bw) = \frac{1}{M}\sum_{m=1}^M h_{\bx,\bx'}(\bw_m)$.
For any estimator $\mu$, which is independent of $\bx$ and $\bw$, there exists $0\leq \alpha <1$ such that 
$\tk(\cdot) = \alpha\mu + (1-\alpha)\hk(\cdot)$ is an estimator with lower risk $\Rcal(k, \tk)<\Rcal(k, \hk)$.
\end{theorem}

\begin{proof}
The proof closely follows ~\citet{MuandetFSGS2013} for a different problem under the non-parametric setting
\footnote{The original Stein effect has a Gaussian distribution assumption.}.
By bias-variance decomposition, we have 
\[
\renewcommand{\arraystretch}{1.5}
\begin{array}{ccl}
	\Rcal(k, \hk) & = & \bE_{\bx, \bx'}\big[\mbox{Var}_\bw(\hk(\bx-\bx',\bw))\big]\\ 
	\Rcal(k, \tk) & = & (1-\alpha)^2\bE_{\bx, \bx'}\big[\mbox{Var}_\bw(\hk(\bx- \bx',\bw))\big] \\
				& & + \alpha^2\bE_{\bx, \bx'}\big[\left(\mu-k(\bx-\bx')\right)^2\big].
\end{array}
\]
By elementary calculation, we have $\Rcal(k, \tk)<\Rcal(k, \hk)$ when 
\[
	0 \leq \alpha \leq \frac{2\bE_{\bx, \bx'}\big[\mbox{Var}_\bw(\hk(\bx-\bx',\bw))\big]}{E_{\bx, \bx'}\big[\mbox{Var}_\bw(\hk(\bx-\bx',\bw))\big] + \bE_{\bx, \bx'}\big[\left(\mu-k(\bx-\bx')\right)^2\big]}
\]
and the optimal value happened at  
\[
	\alpha^* = \frac{\bE_{\bx, \bx'}\big[\mbox{Var}_\bw(\hk(\bx-\bx',\bw))\big]}{E_{\bx, \bx'}\big[\mbox{Var}_\bw(\hk(\bx-\bx',\bw))\big] + \bE_{\bx, \bx'}\big[\left(\mu-k(\bx-\bx')\right)^2\big]} < 1
\]
\end{proof}

\begin{corollary} (Shrinkage estimator)
If $k(\bx-\bx')>0$, there is a shrinkage estimator $\tk(\cdot) = (1-\alpha^*)\hk(\cdot)$ that has lower risk than the empirical mean estimator 
$\hk(\cdot)$ by setting $\mu=0$.
\end{corollary}

The standard shrinkage estimator derived from Stein effect uses the equal weights $\frac{1-\alpha}{M}$ to minimize the variance; however, a non-uniform weights usually results in better performance in practice~\cite{Huszar_UAI_12}. Therefore, based on the analysis of Theorem~\ref{thm:stein}, we proposed the estimator $\tk(\bx-\bx',\bw)=\sum_{m=1}^M \beta_m h(\bw_m)$ obtained  by minimizing the empirical risk with the constraint $\sum_{m=1}^M \beta_m^2 \leq c$ to shrink the weights, where $c$ is a small constant. With Lagrangian multiplier, we end up the following risk minimization formulation which is \emph{data-driven} by directly taking $k(\bx,\bx')$ into account, 
\begin{equation}
	\min_{\bbeta \in \bR^M} \sum_{(\bx,\bx') \sim D} 
    	\Big[ k(\bx,\bx') - [z(\bx)\circ z(\bx')]^T\bbeta \Big]^2 + \lambda_{\beta} \|\bbeta\|^2,
        \label{eq:cvx-obj-beta}
\end{equation}
where $\circ$ denotes the Hadamard product, and $\lambda$ is the regularization coefficient to shrink the weights.
 
The objective function \eqref{eq:cvx-obj-beta} can be further simplified as a least squares regression (LSR) formulation,
\begin{equation}
	\min_{\bbeta \in \bR^{M}} 
    	\| \bt - Z \bbeta \|^2 + \lambda_{\beta} \bbeta^T\bbeta,
        \label{eq:ridge-origin}
\end{equation}
where $\bt \in \bR^{n^2}, Z \in \bR^{n^2 \times M}$, such that for all $i,j \in S=\{1,\ldots,n\}$, there exist a corresponding pair mapping function $\xi: \xi_S(i,j) \rightarrow k$ satisfying $t_k \coloneqq \phi(\bx_i)^T\phi(\bx_j)$ and $Z(k,:) \coloneqq \bz(\bx_i) \circ \bz(\bx_j)$.

We approximately solve optimization problem \eqref{eq:ridge-origin} via the matrix sketching techniques
\cite{avron2013sketching,avron2016faster,woodruff2014sketching}:
\begin{equation}
    \min_{\bbeta \in \bR^{M}} 
    	\| S\bt - SZ \bbeta \|^2 + \lambda_{\beta} \bbeta^T\bbeta,
        \label{eq:ridge-sketch}
\end{equation}
where $S \in \bR^{r\times n^2}$ is the randomized sketching matrix. Solving the sketched LSR problem \eqref{eq:ridge-sketch} now
costs only $O(rd^2 + T_s)$ where $T_s$ is the time cost of sketching. From the optimization perspective, suppose
$\bbeta^{*}$ be the optimal solution of LSR problem \eqref{eq:ridge-origin}, and $\tilde{\bbeta}$ be the solution of
sketched LSR problem \eqref{eq:ridge-sketch}. \cite{wang2017sketched} prove that if $r=O(M/\epsilon + poly(M))$,
the the objective function value of \eqref{eq:ridge-sketch} at $\tilde{\bbeta}$ is at most $\epsilon$ worse than the
value of \eqref{eq:ridge-origin} at $\bbeta^{*}$.
In practice, we consider an efficient $T_s=O(r)$ sampling-based sketching matrix where $S$ is a diagonal matrix with $S_{i,i} = 1/p_i$
if row $i$ was include in the sample, and $S_{i,i} = 0$ otherwise. That being said, we form a sub-matrix of 
$Z$ by including each row of $Z$ in the sub-matrix independently with probability $p_i$. \cite{woodruff2014sketching} 
suggests to set $p_i$ proportional to $\|Z_i\|$, the norm of $i$th row of $Z$, to have a meaningful error bound.

\iffalse
The closed form solution for ridge regression is 
\begin{equation}
	\bbeta^* = \big( \bZ^T\bZ + \lambda_{\beta}\bI_{M} \big)^{-1}\bZ^T\bt.
\end{equation}

\paragraph{Remark}
At first glance, the computational complexity is enormous because of the $n^2$ pairs of rows in $\bZ$. One naive approach is used stochastic gradient based algorithm to solve, which requires proper tuning of the learning rate. However, $\bZ^T\bZ$ is actually a small matrix whose rank is not greater than $M$, the number of random features. From this perspective, it is very attracting to perform sub-sampling on a subset of training data point to form a low-rank approximation of $\bZ^T\bZ$. Suppose we sample a subset $s \subset S$, and our stochastic estimator becomes
\begin{equation}
	\hat{\bbeta} 
    = \Big( \bZ_{\xi_s}^T \bZ_{\xi_s} + \lambda_{\beta}\bI_M \Big)^{-1}
    	\bZ_{\xi_s}^T \bt_{\xi_s}.
    \label{eq:shrinkage-weight}
\end{equation}
\fi

%% algorithm box for proposed method
\begin{algorithm}
	\caption{Data-driven random feature (SES)}
	\label{alg:proposed-alg}
	\begin{algorithmic}[1]
		\Procedure{SES}{}
		\State Compute the Fourier transform $p$ of the kernel  
		\State Generate a sequence $\bw_1,\ldots,\bw_M$ by QMC or MC
		\State Compute un-weighted random feature $z(\bx)$ by \eqref{eq:rf}
		\State Calculate shrinkage weight by \eqref{eq:ridge-sketch}
		\State Compute data-driven random features by 
			\begin{equation*}
    			z'(\bx) = diag(\sqrt{\bbeta}) z(\bx)
    		\end{equation*}
		\EndProcedure
	\end{algorithmic}
\end{algorithm}

%\begin{itemize}
%	\item \peter{@CL, I am thinking of putting Stein Effect and Our proposed method in this section,
%    while moving the BQ as related work in section 4. What do you think ?}
%\end{itemize}
%
%\CL{I will do that tonight}
%\begin{itemize}
%	\item Stein effect
%    \item MMD objective
%    \item prior, regularization, L2 and L1 reg
%    \item stochastic optimization
%\end{itemize}

\section{Experiments}
\label{sec:exp}
\begin{figure*}[!ht]
    \centering
	\begin{subfigure}[b]{\DoubleWidth}
        \includegraphics[width=\textwidth]{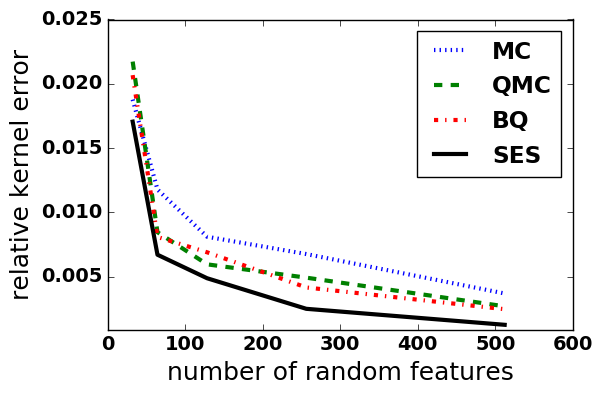}
		\caption{cpu}
		\label{fig:cpu-kdiff}
	\end{subfigure}
	\begin{subfigure}[b]{\DoubleWidth}
        \includegraphics[width=\textwidth]{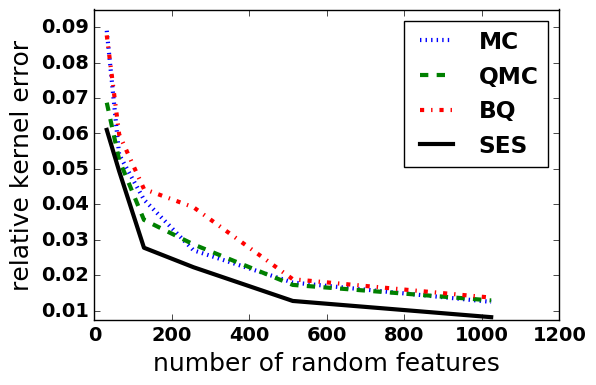}
		\caption{census}
		\label{fig:census-kdiff}
	\end{subfigure}
	\begin{subfigure}[b]{\DoubleWidth}
		\includegraphics[width=\textwidth]{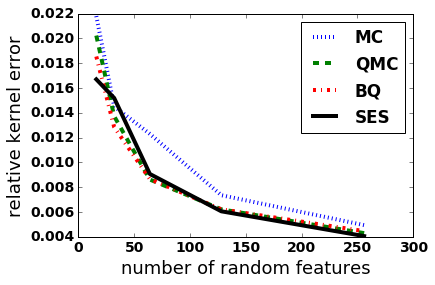}
		\caption{years}
		\label{fig:years-kdiff}
	\end{subfigure}
	
	\begin{subfigure}[b]{\DoubleWidth}
		\includegraphics[width=\textwidth]{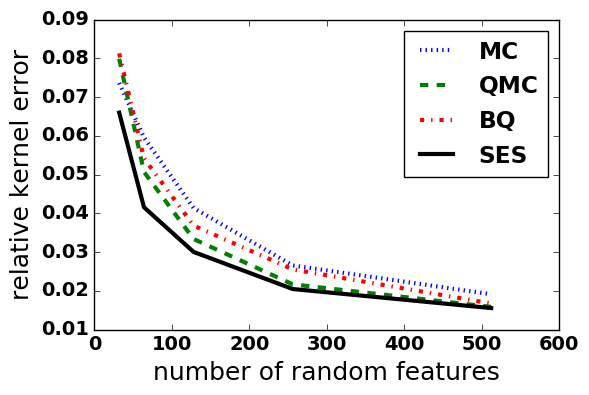}
        \caption{adult}
		\label{fig:adult-kdiff}
	\end{subfigure}
	\begin{subfigure}[b]{\DoubleWidth}
        \includegraphics[width=\textwidth]{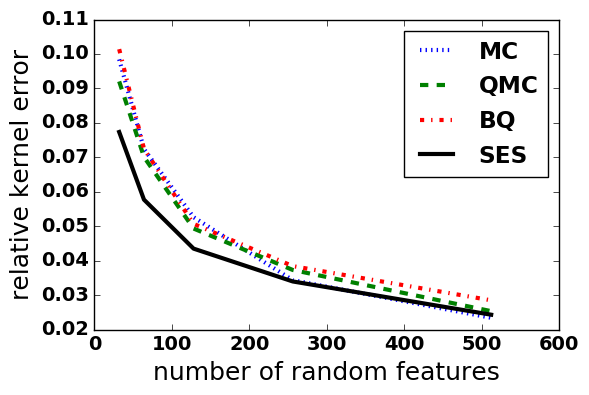}
		\caption{mnist}
		\label{fig:mnist-kdiff}
	\end{subfigure}
	\begin{subfigure}[b]{\DoubleWidth}
		\includegraphics[width=\textwidth]{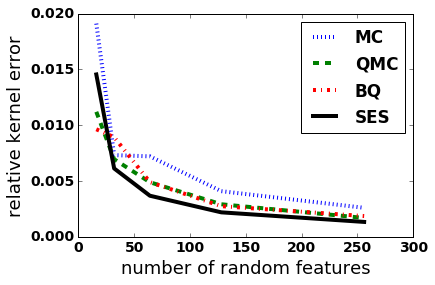}
		\caption{covtype}
		\label{fig:covtype-kdiff}
	\end{subfigure}
	\vspace{-0.25cm}
	\caption{Kernel approximation results. $x$-axis is number of samples
	used to approximate the integral and $y$-axis shows the relative 
	kernel approximation error.}
	\label{fig:kernel-approx}
    \vspace{-0.5cm}
\end{figure*}

\begin{table}[!ht]
	\begin{center}
  	\begin{tabular}{lrrr} 
		\toprule
		Dataset & \text{Task} & \text{n} & \text{d}\\
		\midrule
		CPU & regression & 6554 & 21 \\
		Census & regression & 18,186 & 119 \\
        Years & regression & 463,715 & 90 \\
		Adult & classification & 32,561 & 123 \\
		MNIST & classification & 60,000 & 778 \\
        Covtype & classification & 581,012 & 54 \\
		\bottomrule
	\end{tabular}
	\end{center}
  	\caption{Data Statistics. $n$ is number of instances and $d$ is 
        dimension of instances.}
	\label{tb:data-stats}
    \vspace{-0.25cm}
\end{table}

In this section, we empirically demonstrate the benefits of our proposed method on six data sets listed in Table \ref{tb:data-stats}. The features in all the six data sets are scaled to zero mean and unit variance in a preprocessing step.

We consider two different tasks: i) the quality of kernel approximation and ii) the applicability in supervised learning. Both the relative kernel approximation errors and the regression/classification 
errors are reported on the test sets. 

\subsection{Methods for Comparison}
\begin{itemize}
	\item MC: Monte Carlo approximation with uniform weight. 
    \item QMC: Quasi-Monte Carlo approximation with uniform weight. 
    	We adapt the implementation of \citet{Yang_ICML_14} on Github. 
        \footnote{https://github.com/chocjy/QMC-features}
    \item BQ: QMC sequence with Bayesian Quadrature weights. We 
    	modify the source code from \citet{Huszar_UAI_12}. 
        \footnote{http://www.cs.toronto.edu/~duvenaud/}
\end{itemize}

Gaussian RBF kernel is used through all the experiments where the kernel bandwidth $\sigma$ is tuned on the set $\{2^{-10}, 2^{-8}, \ldots, 2^{8}, 2^{10} \}$ that is in favor of Monte Carlo methods. For QMC method, we use scrambling and shifting techniques recommended in the QMC literature (See \citet{Dick_Acta_13} for more details). In BQ framework, we consider the GP prior with the covariance matrix $k_{GP}(\bw,\bw') = \exp(-\frac{\|\bw-\bw'\|_2^2}{2\sigma_{GP}^2})$ where $\sigma_{GP}$ is tuned on the set $\{2^{-8}, 2^{-6}, \ldots, 2^{6}, 2^{8} \}$ over a subset of sampled pair data $(\bx,\bx')$. Likewise, regularization coefficient $\lambda_{\beta}$ of in the proposed objective function \eqref{eq:cvx-obj-beta} is tuned on the set $\{2^{-8}, 2^{-6}, \ldots, 2^{6}, 2^{8} \}$ over the same subset of sampled pair data. For regression task, we solve ridge regression problems where the regularization coefficient $\lambda$ is tuned by five folds cross-validation. For classification task, we solve L2-loss SVM where the regularization coefficient $C$ is tuned by five folds cross-validation. All experiments code are written in MATLAB and run on a Intel(R) Xeon(R) CPU 2.40GHz Linux server with 32 cores and 190GB memory.

\subsection{Quality of Kernel Approximation}
In our setting, the most natural metric for comparison is the quality of approximation of the kernel matrix. We use the relative kernel approximation error $\|K - \tilde{K}\|_F / \|K\|_F$ to measure the quality, as shown in Figure \ref{fig:kernel-approx}.

SES outperforms the other methods on cpu, census, adult, mnist, covtype, and is competitive with the best baseline methods over the years data sets. Notice that SES is particularly strong when the number of random features is relatively small. This phenomenon confirms our theoretical analysis, that is, the smaller the random features are, the larger the variance would be in general. Our shrinkage estimator with Stein effect enjoys better performance than the empirical mean estimators (MC and QMC) using uniformly weighted features without shrinkage. Although BQ does use weighted random features, its performance is still not as good as SES because tuning the covariance matrix in the Gaussian Process is rather difficult. Moreover, it is even harder to tune a global bandwidth for all integrand functions that approximate the kernel function.

% Our proposed weighting techniques on QMC sequence show manifest improvement over
% other state-of-the-art methods on cpu, census, adult, and covtype data sets. While
% for years and mnist data set, it is as good as other strong baselines. We also observe 
% that our method enjoys better kernel approximation when number of random features
% is small. This phenomenon confirms our theoretical analysis that, 
% the smaller the random features the larger the variance, 
% our shrinkage estimator with Stein effect would enjoy better performance 
% compared to empirical mean estimator (i.e. QMC with uniform weights which sum to 1). On the other hand,
% BQ weights with QMC sequence does not perform very well in general. We address this issue to
% BQ framework is highly sensitive to the Gaussian Process covariance matrix, 
% and it is hard to tune a global bandwidth for all integrand functions that approximate 
% the kernel function well.

\subsection{Supervised Learning}
For regression task, we report the relative regression error $\|\by - \tilde{\by}\|_2 / \|\by\|_2$. For classification task, we report the classification error. 

We investigate if better kernel approximations reflect better predictions in the supervised learning tasks. The results are shown in Table \ref{tb:y-error}. 
\begin{table}[!ht]
	\begin{center}
  	\begin{tabular}{lrrrrr} 
		\toprule
		Dataset & $M$ & MC & QMC & BQ & SES\\
		\midrule
        cpu		& 512 &          3.35\% &	 3.29\% &         3.29\% &\textbf{ 3.27\%} \\
		census	& 256 &			 8.46\% &    6.60\% &\textbf{6.44\%} &			6.49\% \\
		years	& 128 &			0.474\% &	0.473\% &		 0.473\% &\textbf{0.473\%} \\
		\midrule
		adult	& 64  &		    16.21\% &	15.57\% &        15.87\% &\textbf{15.45\%} \\
		mnist	& 256 &\textbf{ 7.33\%} &	 7.72\% &         7.96\% &          7.71\% \\
		covtype	& 128 &		    23.36\% &	22.88\% &        23.13\% &\textbf{22.67\%} \\
		\bottomrule
	\end{tabular}
	\end{center}
    \vspace{-0.25cm}
    \caption{Supervised learning errors. $M$ is number of random features. Error rate is 
    reported, the lower the better.}
    %Regression error and classification errors are reported for 
    %regression task and classification task, respectively.}
	\label{tb:y-error}
    \vspace{-0.5cm}
\end{table}
We can see that among the same data set (cpu, census, adult, and covtype) our method performs very well in terms of kernel approximation, and it also yields lower errors in most of the cases of the supervised learning tasks. Note that we only present partial experiment results on some selected number of random features. However, we also observe that the generalization error of our proposed method may be higher than that of other state-of-the-art methods at some other number of random features that we did not present here. This situation is also observed in other works~\cite{Avron_JMLR_16}: better kernel approximation does not always result in better supervised learning performance. Under certain conditions, less number of features (worse kernel approximation) can play a similar role of regularization to avoid overfitting~\cite{rudi2016genrf}. \citet{rudi2016genrf} suggest to solve this issue by tuning number of random features.

\subsection{Distributions of Feature Weights}
% We are interested in the distribution of weights derived from BQ and 
% our proposed method. For different number of random features 
% ($M=32,128,512$), we plot the histogram of weights and compared it
% with the equal weight $1/M$ used in MC and QMC, 
% as shown in Figure \ref{fig:wgt-hist}.
% When the number random features is small, the variance is big. As the 
% number of random features increase, the weights gradually converge to
% a stable distribution close to the uniform weight spike. 
% his empirical
% observation agrees with theoretical results we found, namely with the
% role of the weight estimator in the bias-variance trade-off.

We are interested in the distributions of random feature weights derived from BQ and SES. For different numbers of random features ($M=32,128,512$), we plotted the corresponding histograms and compared them with the equal weight of $1/M$ used in MC and QMC, as shown in Figure \ref{fig:wgt-hist}. When the number of random features is small, the difference between BQ and SES is big. As the number of random features increases, the weights of each method gradually converge to a stable distribution, and get increasingly closer to the uniform weight spike. This empirical observation agrees with the theoretical analysis we found, namely with the role of the weight estimator in the bias-variance trade-off.
\begin{figure}[!ht]
    \centering
	\begin{subfigure}[b]{\SingleWidth}
        \includegraphics[width=\textwidth]{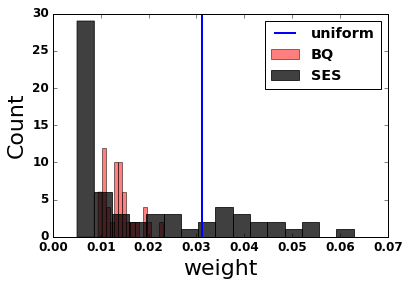}
		\caption{$M=32$}
		\label{fig:dim16-wgt}
	\end{subfigure}
	\begin{subfigure}[b]{\SingleWidth}
        \includegraphics[width=\textwidth]{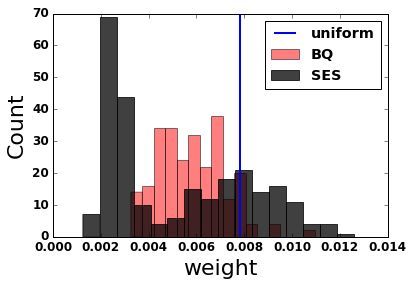}
		\caption{$M=128$}
		\label{fig:dim64-wgt}
	\end{subfigure}
	
	\begin{subfigure}[b]{\SingleWidth}
        \includegraphics[width=\textwidth]{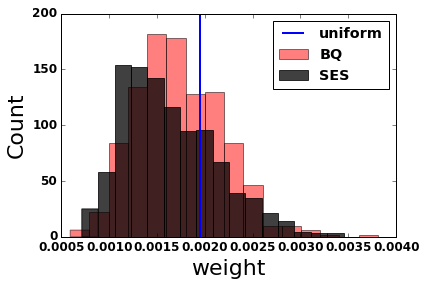}
		\caption{$M=512$}
		\label{fig:dim256-wgt}
	\end{subfigure}
	\begin{subfigure}[b]{\SingleWidth}
        \includegraphics[width=\textwidth]{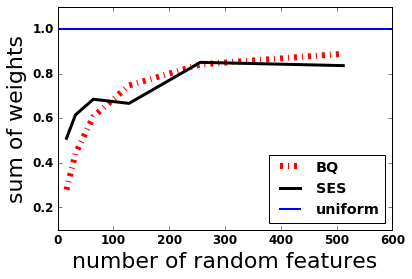}
		\caption{sum of weights versus $M$}
		\label{fig:sum-wgt}
	\end{subfigure}
	
    \vspace{-0.15cm}
	\caption{A comparison of distribution of weights on adult data set. $M$ is
    the number of random features.}
	\label{fig:wgt-hist}
    \vspace{-0.3cm}
\end{figure}

\subsection{Behavior of Randomized Optimization Solver}
In theory, our randomized solver only needs to sample $O(M)$ data in our proposed objective function \eqref{eq:ridge-sketch} to have a sufficiently accurate solution. We verify this assertion by varying the sampled data size $r$ in sketching matrix $S$ and examine it affects the relative kernel approximation errors. 
The results are shown in Figure \ref{fig:subset}. For census data set, we observe that as the number of sampled data $r$ exceeds the number of random features $M$, the kernel approximation of our proposed method outperforms the state-of-the-art, QMC. Furthermore, the number of sampled data required by our proposed randomized solver is still the same order of $M$ even for a much larger data set covtype, with $0.5$ millions data points. On the other hand, BQ does not have consistent behavior regarding the sampled data size, which is used for tuning the hyperparmaeter $\sigma_{GP}$. This again illustrates that BQ is highly sensitive to the covariance matrix and is difficult to tune $\sigma_{GP}$ on only a small subset of sampled data. In practice, we found that we usually only need to sample up to $2\sim4$ times the number of random features $M$ to get a relatively stable and good minimizer $\bbeta^*$. This observation coincides with the complexity analysis of our proposed randomized solver in Section \ref{sec:Stein}.

\begin{figure}[!ht]
    \centering
	\begin{subfigure}[b]{\SingleWidth}
        \includegraphics[width=\textwidth]{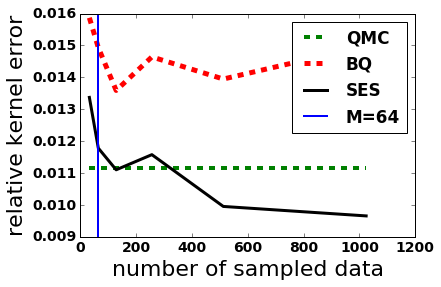}
		\caption{$M=64$, census data set}
		\label{fig:census_subset_64}
	\end{subfigure}
	\begin{subfigure}[b]{\SingleWidth}
        \includegraphics[width=\textwidth]{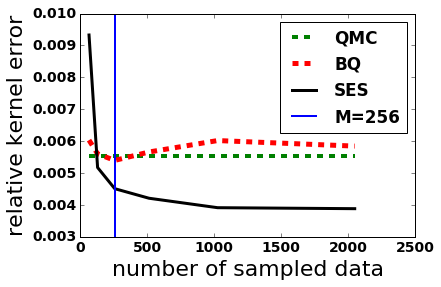}
		\caption{$M=256$, census data set}
		\label{fig:census_subset_256}
	\end{subfigure}
    \\
   	\begin{subfigure}[b]{\SingleWidth}
        \includegraphics[width=\textwidth]{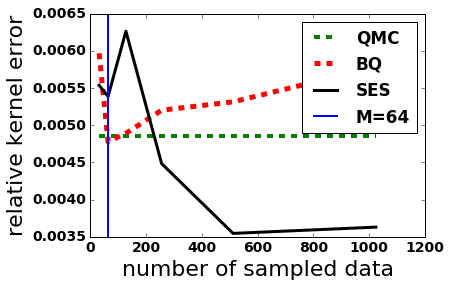}
		\caption{$M=64$, covtype data set}
		\label{fig:covtype_subset_128}
	\end{subfigure}
	\begin{subfigure}[b]{\SingleWidth}
        \includegraphics[width=\textwidth]{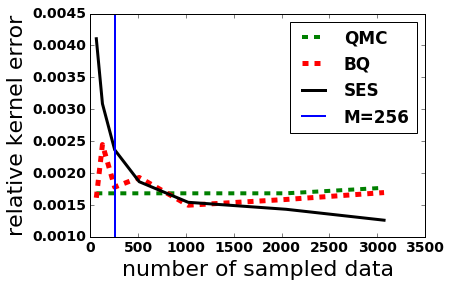}
		\caption{$M=256$, covtype data set}
		\label{fig:covtype_subset_256}
	\end{subfigure}
	
    \vspace{-0.1cm}
	\caption{A comparison of different sampled size training data in solving
    objective function \eqref{eq:ridge-sketch}.} 
	\label{fig:subset}
    %\vspace{-0.3cm}
\end{figure}

\subsection{Uniform v.s. Non-uniform Shrinkage Weight}
We conducted experiments with SES using a uniform shrinkage weight versus using non-uniform shrinkage weights on the cpu and census data sets. The results in Figures \ref{fig:umf} confirm our conjecture that non-uniform shrinkage weights are more beneficial, although both enjoy the lower empirical risk from the Stein effect. We also observed similar behavior on other data sets, and we omit those results. 
\begin{figure}[!ht]
    \centering
	\begin{subfigure}[b]{\SingleWidth}
        \includegraphics[width=\textwidth]{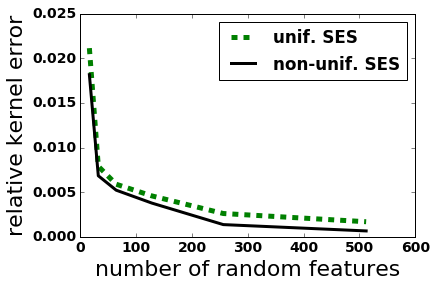}
		\caption{cpu data set}
		\label{fig:cpu_umf_kdiff}
	\end{subfigure}
	\begin{subfigure}[b]{\SingleWidth}
        \includegraphics[width=\textwidth]{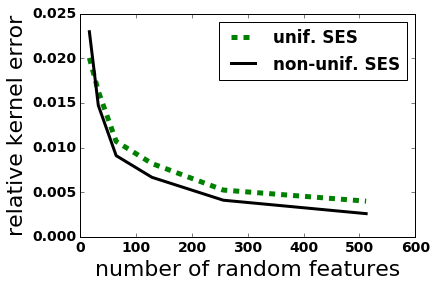}
		\caption{census data set}
		\label{fig:census_umf_kdiff}
	\end{subfigure}
	
    \vspace{-0.1cm}
    \caption{Comparison of uniform and non-uniform shrinkage weight.} 
	\label{fig:umf}
    \vspace{-0.25cm}
\end{figure}

\section{Discussion and Conclusion}
Nystr{\"o}m methods \cite{Williams_NIPS_01,Drineas_JMLR_05} is yet another line of research for kernel approximation. Specifically, eigendecomposition of the kernel matrix is achieved via low rank factorization with various sampling techniques \cite{wang2013improving}. We refer interested readers to the comprehensive study \cite{yang2012nystrom} for more detailed comparisons. Due to the page limits, the scope of this paper only focus on the study of random Fourier features. 

Finally, we notice \cite{sinha2016learning} proposes a weighting scheme for random features which matches the label correlation matrix with the goal to optimize the supervise objective (e.g. classification errors). In contrast, our work aims to optimize the kernel approximation, which does not require the label information to solve weights and can be applied to different tasks after solving the weights once. %In short, our method is like an unsupervised counterpart of \cite{sinha2016learning}.

To conclude, we propose a novel shrinkage estimator, which enjoys the benefits of Stein effect w.r.t. lowering the empirical risk compared to empirical mean estimators. Our estimator induces non-uniform weights on the random features with desirable data-driven properties. We further provide an efficient randomized solver to optimize this estimation problem for real-world large-scale applications. The extensive experimental results demonstrate the effectiveness of our proposed method.

\section*{Acknowledgments}
We thank the reviewers for their helpful comments. This work is supported in
part by the National Science Foundation (NSF) under grants IIS-1546329 and IIS-1563887.

\bibliographystyle{named}
\bibliography{main}

\begin{thebibliography}{}

\bibitem[\protect\citeauthoryear{Avron \bgroup \em et al.\egroup
  }{2013}]{avron2013sketching}
Haim Avron, Vikas Sindhwani, and David Woodruff.
\newblock Sketching structured matrices for faster nonlinear regression.
\newblock In {\em NIPS}, 2013.

\bibitem[\protect\citeauthoryear{Avron \bgroup \em et al.\egroup
  }{2016a}]{avron2016faster}
Haim Avron, Kenneth~L Clarkson, and David~P Woodruff.
\newblock Faster kernel ridge regression using sketching and preconditioning.
\newblock {\em arXiv}, 2016.

\bibitem[\protect\citeauthoryear{Avron \bgroup \em et al.\egroup
  }{2016b}]{Avron_JMLR_16}
Haim Avron, Vikas Sindhwani, Jiyan Yang, and Michael~W Mahoney.
\newblock Quasi-monte carlo feature maps for shift-invariant kernels.
\newblock {\em JMLR}, 2016.

\bibitem[\protect\citeauthoryear{Bach}{2015}]{Bach_arXiv_15}
Francis Bach.
\newblock On the equivalence between quadrature rules and random features.
\newblock {\em arXiv}, 2015.

\bibitem[\protect\citeauthoryear{Caflisch}{1998}]{Caflisch_Acta_98}
Russel~E Caflisch.
\newblock Monte carlo and quasi-monte carlo methods.
\newblock {\em Acta numerica}, 1998.

\bibitem[\protect\citeauthoryear{Chang and Lin}{2011}]{Chang_TIST_11}
Chih-Chung Chang and Chih-Jen Lin.
\newblock Libsvm: a library for support vector machines.
\newblock {\em ACM Transactions on Intelligent Systems and Technology (TIST)},
  2011.

\bibitem[\protect\citeauthoryear{Chitta \bgroup \em et al.\egroup
  }{2012}]{Chitta_ICDM_12}
Radha Chitta, Rong Jin, and Anil~K Jain.
\newblock Efficient kernel clustering using random fourier features.
\newblock In {\em ICDM}, 2012.

\bibitem[\protect\citeauthoryear{Dai \bgroup \em et al.\egroup
  }{2014}]{Dai_NIPS_14}
Bo~Dai, Bo~Xie, Niao He, Yingyu Liang, Anant Raj, Maria-Florina~F Balcan, and
  Le~Song.
\newblock Scalable kernel methods via doubly stochastic gradients.
\newblock In {\em NIPS}, 2014.

\bibitem[\protect\citeauthoryear{Dick \bgroup \em et al.\egroup
  }{2013}]{Dick_Acta_13}
Josef Dick, Frances~Y Kuo, and Ian~H Sloan.
\newblock High-dimensional integration: the quasi-monte carlo way.
\newblock {\em Acta Numerica}, 2013.

\bibitem[\protect\citeauthoryear{Drineas and Mahoney}{2005}]{Drineas_JMLR_05}
Petros Drineas and Michael~W Mahoney.
\newblock On the nystr{\"o}m method for approximating a gram matrix for
  improved kernel-based learning.
\newblock {\em JMLR}, 2005.

\bibitem[\protect\citeauthoryear{Ghahramani and
  Rasmussen}{2002}]{Ghahramani_NIPS_02}
Zoubin Ghahramani and Carl~E Rasmussen.
\newblock Bayesian monte carlo.
\newblock In {\em NIPS}, 2002.

\bibitem[\protect\citeauthoryear{Huang \bgroup \em et al.\egroup
  }{2014}]{Huang_ICASSP_14}
Po-Sen Huang, Haim Avron, Tara~N Sainath, Vikas Sindhwani, and Bhuvana
  Ramabhadran.
\newblock Kernel methods match deep neural networks on timit.
\newblock In {\em ICASSP}, 2014.

\bibitem[\protect\citeauthoryear{Husz{\'a}r and Duvenaud}{2012}]{Huszar_UAI_12}
Ferenc Husz{\'a}r and David Duvenaud.
\newblock Optimally-weighted herding is bayesian quadrature.
\newblock In {\em UAI}, 2012.

\bibitem[\protect\citeauthoryear{Kivinen \bgroup \em et al.\egroup
  }{2004}]{Kivinen_SP_04}
Jyrki Kivinen, Alexander~J Smola, and Robert~C Williamson.
\newblock Online learning with kernels.
\newblock {\em IEEE transactions on signal processing}, 52(8), 2004.

\bibitem[\protect\citeauthoryear{Le \bgroup \em et al.\egroup
  }{2013}]{Le_ICML_13}
Quoc Le, Tam{\'a}s Sarl{\'o}s, and Alexander Smola.
\newblock Fastfood-computing hilbert space expansions in loglinear time.
\newblock In {\em ICML}, 2013.

\bibitem[\protect\citeauthoryear{Li and P{\'o}czos}{2016}]{Li_UAI_16}
Chun-Liang Li and Barnab{\'a}s P{\'o}czos.
\newblock Utilize old coordinates: Faster doubly stochastic gradients for
  kernel methods.
\newblock In {\em UAI}, 2016.

\bibitem[\protect\citeauthoryear{Muandet \bgroup \em et al.\egroup
  }{2014}]{MuandetFSGS2013}
K.~Muandet, K.~Fukumizu, B.~Sriperumbudur, A.~Gretton, and B.~Sch{\"o}lkopf.
\newblock Kernel mean estimation and stein effect.
\newblock In {\em ICML}, 2014.

\bibitem[\protect\citeauthoryear{Rahimi and Recht}{2007}]{Rahimi_NIPS_07}
Ali Rahimi and Benjamin Recht.
\newblock Random features for large-scale kernel machines.
\newblock In {\em NIPS}, 2007.

\bibitem[\protect\citeauthoryear{Rahimi and Recht}{2009}]{Rahimi_NIPS_09}
Ali Rahimi and Benjamin Recht.
\newblock Weighted sums of random kitchen sinks: Replacing minimization with
  randomization in learning.
\newblock In {\em NIPS}, 2009.

\bibitem[\protect\citeauthoryear{Rasmussen}{2006}]{Rasmussen_GP_06}
Carl~Edward Rasmussen.
\newblock Gaussian processes for machine learning.
\newblock 2006.

\bibitem[\protect\citeauthoryear{Rudi \bgroup \em et al.\egroup
  }{2016}]{rudi2016genrf}
Alessandro Rudi, Raffaello Camoriano, and Lorenzo Rosasco.
\newblock Generalization properties of learning with random features.
\newblock {\em arXiv}, 2016.

\bibitem[\protect\citeauthoryear{Rudin}{2011}]{Rudin_book_11}
Walter Rudin.
\newblock {\em Fourier analysis on groups}.
\newblock John Wiley \& Sons, 2011.

\bibitem[\protect\citeauthoryear{Sinha and Duchi}{2016}]{sinha2016learning}
Aman Sinha and John~C Duchi.
\newblock Learning kernels with random features.
\newblock In {\em NIPS}, 2016.

\bibitem[\protect\citeauthoryear{Steinwart and
  Christmann}{2008}]{Steinwart_book_08}
Ingo Steinwart and Andreas Christmann.
\newblock {\em Support vector machines}.
\newblock Springer Science \& Business Media, 2008.

\bibitem[\protect\citeauthoryear{Wang and Zhang}{2013}]{wang2013improving}
Shusen Wang and Zhihua Zhang.
\newblock Improving cur matrix decomposition and the nystr{\"o}m approximation
  via adaptive sampling.
\newblock {\em JMLR}, 2013.

\bibitem[\protect\citeauthoryear{Wang \bgroup \em et al.\egroup
  }{2017}]{wang2017sketched}
Shusen Wang, Alex Gittens, and Michael~W Mahoney.
\newblock Sketched ridge regression: Optimization perspective, statistical
  perspective, and model averaging.
\newblock {\em arXiv}, 2017.

\bibitem[\protect\citeauthoryear{Williams and Seeger}{2001}]{Williams_NIPS_01}
Christopher Williams and Matthias Seeger.
\newblock Using the nystr{\"o}m method to speed up kernel machines.
\newblock In {\em NIPS}, 2001.

\bibitem[\protect\citeauthoryear{Woodruff and
  others}{2014}]{woodruff2014sketching}
David~P Woodruff et~al.
\newblock Sketching as a tool for numerical linear algebra.
\newblock {\em Foundations and Trends{\textregistered} in Theoretical Computer
  Science}, 2014.

\bibitem[\protect\citeauthoryear{Yang \bgroup \em et al.\egroup
  }{2012}]{yang2012nystrom}
Tianbao Yang, Yu-Feng Li, Mehrdad Mahdavi, Rong Jin, and Zhi-Hua Zhou.
\newblock Nystr{\"o}m method vs random fourier features: A theoretical and
  empirical comparison.
\newblock In {\em NIPS}, 2012.

\bibitem[\protect\citeauthoryear{Yang \bgroup \em et al.\egroup
  }{2014}]{Yang_ICML_14}
Jiyan Yang, Vikas Sindhwani, Haim Avron, and Michael~W Mahoney.
\newblock Quasi-monte carlo feature maps for shift-invariant kernels.
\newblock In {\em ICML}, 2014.

\end{thebibliography}

\end{document}